\documentclass[11pt]{article}

\def\colorful{1}

\usepackage{enumitem}
\usepackage{dsfont}
\usepackage{subfig}
\usepackage{tikz}
\usepackage{pgf,pgfplots}
\usetikzlibrary{angles, quotes, patterns}
\usepackage{pgfplots}

\usepackage{amsthm,amsfonts,amsmath,amssymb,epsfig,color,float,graphicx,verbatim,enumitem}
\usepackage{algorithm}
\usepackage[noend]{algpseudocode}
\usepackage{environ}

\usepackage{hyperref}
\hypersetup{
  colorlinks = true,
  urlcolor = {blue},
  citecolor = {blue}
}
\definecolor{cpurple}{rgb}{0.6,0,0.6}

\def\nnewcolor{0}
\ifnum\nnewcolor=1
\newcommand{\nnew}[1]{{\color{red} #1}}
\fi
\ifnum\nnewcolor=0
\newcommand{\nnew}[1]{#1}
\fi

\ifnum\colorful=1

\else

\fi

\newtheorem{theorem}{Theorem}[section]

\newtheorem{lemma}[theorem]{Lemma}
\newtheorem{informal theorem}[theorem]{Theorem (informal statement)}

\newtheorem{claim}[theorem]{Claim}

\theoremstyle{definition}
\newtheorem{definition}[theorem]{Definition}
\newcommand{\eqdef}{\stackrel{{\mathrm {\footnotesize def}}}{=}}

\newtheorem{innercustomgeneric}{\customgenericname}
\providecommand{\customgenericname}{}
\newcommand{\newcustomtheorem}[2]{\newenvironment{#1}[1]
	{\renewcommand\customgenericname{#2}\renewcommand\theinnercustomgeneric{##1}\innercustomgeneric
	}
	{\endinnercustomgeneric}
}

\newcustomtheorem{customthm}{Theorem}
\newcustomtheorem{customlem}{Lemma}
\newcustomtheorem{customclm}{Claim}
\newcustomtheorem{customfc}{Fact}

\newcommand{\sigMO}[1]{ e^{-\frac{#1}{\sigma }}}
\newcommand{\density}{\gamma}

\newcommand\snorm[2]{\left\| #2 \right\|_{#1}}
\renewcommand\vec[1]{\mathbf{#1}}
\DeclareMathOperator*{\Prob}{\mathbf{Pr}}
\DeclareMathOperator*{\E}{\mathbf{E}}
\newcommand{\proj}{\mathrm{proj}}

\newcommand{\mrm}{\mathrm}
\newcommand{\SL}{\mathcal{L}_{\sigma}}

\def\d{\mathrm{d}}

\DeclareMathOperator*{\argmin}{argmin}

\newcommand{\bx}{\mathbf{x}}
\newcommand{\by}{\mathbf{y}}

\newcommand{\bw}{\mathbf{w}}

\newcommand{\err}{\mathrm{err}}

\newcommand{\R}{\mathbb{R}}

\newcommand{\Z}{\mathbb{Z}}

\newcommand{\eps}{\epsilon}

\newcommand{\pr}{\mathbf{Pr}}
\newcommand{\poly}{\mathrm{poly}}

\newcommand{\sgn}{\mathrm{sign}}
\newcommand{\sign}{\mathrm{sign}}

\newcommand{\calC}{{\cal C}}

\newcommand{\opt}{\mathrm{opt}}
\newcommand{\D}{\mathcal{D}}

\newcommand{\Ind}{\mathds{1}}
\newcommand{\1}{\Ind}

\newcommand{\littlesum}{\mathop{\textstyle \sum}}

\newcommand{\wt}{\widetilde}
\newcommand{\wh}{\widehat}

\newcommand{\dotp}[2]{\left\langle #1, #2 \right\rangle}

\newcommand{\wstar}{\bw^{\ast}}
\newcommand{\x}{\vec x}

\newcommand{\citep}{\cite}

\title{Non-Convex SGD Learns Halfspaces with Adversarial Label Noise}

\author{
Ilias Diakonikolas\thanks{Supported by NSF Award CCF-1652862 (CAREER), a Sloan Research Fellowship, and 
a DARPA Learning with Less Labels (LwLL) grant.}\\
University of Wisconsin-Madison\\
{\tt ilias@cs.wisc.edu}\\
\and
Vasilis Kontonis\\
University of Wisconsin-Madison\\
{\tt kontonis@wisc.edu }\\
\and
Christos Tzamos\\ University of Wisconsin-Madison\\
{\tt tzamos@wisc.edu}
\and
Nikos Zarifis\thanks{Supported in part by a DARPA  Learning with Less Labels (LwLL) grant.}\\
University of Wisconsin-Madison\\
{\tt zarifis@wisc.edu}\\
}

\begin{document}

\maketitle

\begin{abstract}
  We study the problem of agnostically learning homogeneous halfspaces in the
distribution-specific PAC model. For a broad family of structured
distributions, including log-concave distributions, we show that non-convex SGD
efficiently converges to a solution with misclassification error
$O(\opt)+\eps$, where $\opt$ is the misclassification error of the best-fitting
halfspace.  In sharp contrast, we show that optimizing any convex surrogate
inherently leads to misclassification error of $\omega(\opt)$, even under
Gaussian marginals.
 \end{abstract}

\setcounter{page}{0}
\thispagestyle{empty}
\newpage

\section{Introduction} \label{sec:intro}

\subsection{Background and Motivation} \label{ssec:background}

Learning in the presence of noisy data is a central challenge in
machine learning.  In this work, we study the efficient learnability
of halfspaces when a fraction of the training labels is adversarially
corrupted. As our main contribution, we show that non-convex SGD
efficiently learns homogeneous halfspaces in the presence of
adversarial label noise with respect to a broad family of
well-behaved distributions, including log-concave distributions.
Before we state our contributions, we provide some background and
motivation for this work.

A (homogeneous) halfspace is any function $f: \R^d \to \{ \pm 1\}$ of
the form $f(\bx) = \sgn(\langle \bw, \bx \rangle)$, where the vector
$\bw \in \R^d$ is called the weight vector of $f$, and the function
$\sgn: \R \to \{\pm 1\}$ is defined as $\sgn(t) = 1$ if $t \geq 0$
and $\sgn(t) = -1$ otherwise.  Halfspaces are arguably the most
fundamental concept class and have been studied since the beginning
of machine learning, starting with the Perceptron
algorithm~\cite{Rosenblatt:58, Novikoff:62}.  In the realizable
setting, halfspaces are efficiently learnable in the
distribution-independent PAC model~\cite{val84} via linear
programming (see, e.g.,~\cite{MT:94}).  On the other hand, in the
agnostic model~\cite{Haussler:92, KSS:94}, even {\em weak}
distribution-independent learning is computationally
intractable~\cite{GR:06, FGK+:06short, Daniely16}.  The
distribution-specific agnostic (or adversarial label noise) setting
-- where the label noise is adversarial but we have some prior
knowledge about the structure of the marginal distribution on
examples -- lies in between these two extremes. In this setting,
computationally efficient noise-tolerant learning algorithms are
known~\cite{KKMS:08, KLS09, ABL17, Daniely15, DKS18a} under various
distributional assumptions.  We start by defining the
distribution-specific agnostic model.

\begin{definition}[Distribution-Specific PAC Learning with Adversarial Label Noise] \label{def:agnostic-ds}
Given i.i.d. labeled examples $(\bx, y)$ from a distribution $\D$ on
$\R^d \times \{\pm 1\}$, such that the marginal distribution
$\D_{\bx}$ is promised to belong in a known family $\mathcal{F}$
but the labels $y$ can be arbitrary,
the goal of the learner is to output a hypothesis $h$
with small misclassification error
$\err_{0-1}^{\D}(h) \eqdef \pr_{(\bx, y) \sim \D}[h(\bx) \neq y]$, compared
to $\opt \eqdef \inf_{g \in \mathcal{C}} \err_{0-1}^{\D}(g)$, where $\mathcal{C}$
is the target concept class.
\end{definition}

\cite{KKMS:08} gave an algorithm that learns halfspaces in this model with
error $\opt+\eps$ under any isotropic log-concave distribution, with sample
complexity and runtime $d^{m(1/\eps)}$, for an appropriate function $m$, which
is at least polynomial.  Moreover, there is evidence that any algorithm that
achieves error $\opt+\eps$ requires time {\em exponential} in $1/\eps$,
even under Gaussian marginals~\cite{DKZ20, GGK20}.
\nnew{Specifically, recent work~\cite{DKZ20, GGK20} obtained 
Statistical Query (SQ) lower bounds of $d^{\poly(1/\eps)}$ for this problem.}

A line of work~\cite{KLS09, ABL17, Daniely15, DKS18a} focused on obtaining
$\poly(d, 1/\eps)$ time algorithms with near-optimal error guarantees.
Specifically, ~\cite{ABL17} gave a polynomial time {\em constant-factor}
approximation algorithm -- i.e., an algorithm with misclassification error of
$C \cdot \opt+\eps$, for some universal constant $C>1$ -- for homogeneous
halfspaces under any isotropic log-concave distribution. More recent
work~\cite{DKS18a} gave an algorithm achieving this error bound for arbitrary
halfspaces under Gaussian marginals.  The algorithms of~\cite{ABL17, DKS18a}
rely on an iterative localization technique and  are quite sophisticated.
Moreover, while their complexity is polynomial, they do not appear to be practical.
The motivation for this work is the design of simple and practical algorithms
for this problem with the same near-optimal error guarantees as these prior
works.

\subsection{Our Contributions} \label{ssec:results}

Our main result is that SGD on a non-convex surrogate of the zero-one loss
solves the problem of learning a homogeneous halfspace with adversarial label
noise when the underlying marginal distribution on the examples is
well-behaved.  As we already mentioned, prior work \cite{ABL17,DKS18a} uses
more complex methods and custom algorithms that run in multiple phases using
multiple passes over the samples.  In contrast, we take a direct optimization
approach and define a single loss function over the space of halfspaces whose
approximate stationary points are near-optimal solutions.  This implies that
\emph{any optimization method} that is guaranteed to converge to stationary
points, for example SGD, will yield a halfspace with error $O(\opt) + \eps$.

Our loss function is a smooth version of the 0-1 loss using a sigmoid function.
In our case, we use the logistic function $S_\sigma(t) = 1/(1+e^{-t/\sigma})$.
Our overall objective is:

\begin{equation} \label{eq:intro_objective}
\SL(\vec w)  =
\E_{(\vec x, y) \sim \D} \left[
S_{\sigma}\left(-y \dotp{\vec w}{\vec x} \right)\right]\,,
\end{equation}
and we optimize it over the unit sphere $\snorm{2}{\bw} = 1$.
We show that, for a broad class of distributions, any stationary point of this
loss function corresponds to a halfspace with near-optimal error. In more detail,
we require that the distribution on the examples is
sufficiently well-behaved (Definition~\ref{def:bounds}) satisfying natural (anti-)concentration properties.

In \cite{DKTZ20}, it was shown that the (approximate)
stationary points of the objective of Equation~\eqref{eq:intro_objective} are
(approximately) optimal halfspaces under Massart noise, which is a milder noise
assumption than adversarial label noise.  Interestingly, our results suggest that
optimizing this objective is a unified approach for learning halfspaces
under label noise, as we show that it works even in the more challenging
adversarial noise setting.

\begin{definition}[Well-behaved distributions] \label{def:bounds}
Let $U, R >0$ be absolute constants and $t: \R_+ \rightarrow \R_+$ be a non-negative function.
An isotropic (i.e., zero mean and identity covariance) distribution $\D_{\bx}$ on $\R^d$
is called well-behaved if for any projection $(\D_{\bx})_V$ of $\D_{\bx}$
onto a $2$-dimensional subspace $V$ the corresponding pdf $\gamma_V$ on $\R^2$
satisfies the following properties:
\begin{enumerate}
  \item  $\gamma_V(\vec x) \geq 1/U$, for all $\vec x \in V$ such that $\snorm{2}{\vec x} \leq R$ (anti-anti-concentration).
\item For all $\x\in V$, we have $\gamma_V(\vec x)\leq t\left(\snorm{2}{\vec
    x } \right) $ and also
    $\sup_{\x \in V}t(\snorm{2}{\x})\leq U$, $\int_{V} t(\snorm{2}{\bx}) \d \bx\leq U$, $
    \int_{V} \snorm{2}{\bx} t(\snorm{2}{\bx}) \d \bx \leq U$
    (anti-concentration and concentration).
\end{enumerate}
\end{definition}
Our class of distributions contains well-known distribution classes such as
Gaussian and log-concave.  In addition to distributions with strong
concentration properties, our results also handle distributions with very weak
concentration such as heavy-tailed distributions. In particular, we handle
distributions whose density function decays only polynomially with the distance
from the origin, see Table~\ref{tab:intro_distribution_classes}.

We use the \emph{non-convex} objective of Equation~\eqref{eq:intro_objective} and SGD to
obtain our main algorithmic result.

\begin{theorem}\label{thm:intro-upper}
  Let $\D$ be a distribution on $\R^d \times\{\pm1\}$ such that the
        marginal $\D_{\bx}$ on $\R^d$ is well-behaved.  Then SGD on the
        objective \eqref{eq:intro_objective} has the following performance guarantee:
        For any $\eps>0$, it draws $m = \wt{O}(d/\eps^4) $ labeled examples from $\D$,
        uses $O(m)$ gradient evaluations, and outputs a hypothesis halfspace with
        misclassification error $O(\opt) + \eps$ with probability at least $99\%$.
\end{theorem}

  Theorem~\ref{thm:intro-upper} gives a simple and practical learning algorithm
  for halfspaces with adversarial label noise with respect to a broad family of
  marginal distributions.

  A natural question is whether the non-convexity of our surrogate
  loss~\eqref{eq:intro_objective} is required.
In many practical settings, convex surrogates of the $0/1$ loss such as Hinge or ReLU loss are
used, see \cite{BJM06} for more choices.  In general, given a convex and
increasing loss $\ell(\cdot)$ the following objective is defined.
\begin{equation}
  \label{eq:intro_convex_objective}
  \mathcal{C}(\bw) = \E_{(\bx, y) \sim \D}[
  \ell(-y \dotp{\bx}{\bw}) ] \,.
\end{equation}
One such convex optimization problem closely related to our non-convex formulation is
\emph{logistic regression}.  In that case, the convex surrogate is simply
$\ell(t) = \log(S_\sigma(t))$ (compare with Equation~\eqref{eq:intro_objective}).

To complement our positive result,
we show that convex surrogates are insufficient for the task at hand.
\emph{In particular, for any convex surrogate objective,
one will obtain a halfspace with error $\omega(\opt)$}.
In more detail, we construct a single noisy distribution
whose $\bx$-marginal is well-behaved such that
optimizing any convex objective over this distribution will yield a halfspace with error
$\omega(\opt)$.  We establish a fine-grained result showing that the misclassification
error of convex objectives degrades as the distributions become more
\emph{heavy tailed}, see Table~\ref{tab:intro_distribution_classes}.

\begin{theorem}
  \label{thm:intro_lower}
  Let $\D_{\bx}$ be the standard normal distribution on $\R^d$.  There exists a
  distribution $\D$ on $\R^d \times\{\pm 1\}$ such that for \emph{every} convex and
  non-decreasing loss $\ell(\cdot)$ the objective $\calC(\bw) = \E_{(\bx, y) \sim \D}[\ell(-y \dotp{\bx}{\bw})]$
  is minimized at some halfspace $h$ with misclassification error $\Omega(\opt \sqrt{\log(1/\opt)})$.
  Moreover, if the marginal $\D_{\bx}$ is allowed to be log-concave (resp. $s$-heavy tailed, $s>2$)
  the error of any minimizer is $\Omega(\opt \log(1/\opt))$ (resp. $\Omega(\opt^{1-1/s})$).
\end{theorem}

In fact, our lower bound result shows a strong statement about convex surrogates: Even under
the nicest distribution possible, i.e., a Gaussian, there is some simple label
noise (flipping the labels of points far from the origin) \emph{that does
not depend on the convex loss $\ell(\cdot)$} such that no convex objective can
achieve $O(\opt)$ error.  This suggests that the shortcoming of convex objectives
is not due to pathological cases and complicated noise distributions that are
designed to fool each specific loss function, but is rather inherent.

\begin{table}[ht]
  \def\arraystretch{1.2}
  \caption{Common \emph{well-behaved} distribution families with their
    corresponding parameters $U, R, t(\cdot)$, see Definition~\ref{def:bounds}.
    The last two columns show the best possible error achievable by convex
    objectives and our non-convex objective of Eq.\eqref{eq:intro_objective}.
}
\label{tab:intro_distribution_classes}
  \centering
  \begin{tabular}{l c c c c}
    \hline \hline
    \textbf{Distribution} & $U, R$ & $t(\bx)$ & Any Convex Loss & Our Loss, Eq.\eqref{eq:intro_objective} \\
    \hline \hline
    Gaussian  & $\Theta(1)$ & $ e^{-\Omega(\snorm{2}{\bx}^2)} $ &
    $\Omega(\opt \sqrt{\log(1/\opt)})$ \scriptsize{[Thm~\ref{thm:intro_lower}]}  & $O(\opt)$ \scriptsize{[Thm~\ref{thm:intro-upper}]}
    \\
    Log-Concave & $\Theta(1)$ & $e^{-\Omega(\snorm{2}{\bx})}$ & $\Omega(\opt \log(1/\opt))$
    \scriptsize{[Thm~\ref{thm:intro_lower}]}
                & $O(\opt)$ \scriptsize{[Thm~\ref{thm:intro-upper}]}
                \\
    $s$-Heavy Tailed, $s>2$ & $\Theta(1)$ & $\frac{O(1)}{(\snorm{2}{\bx}+1)^{2+s}}$
                            & $\Omega(\opt^{1-1/s})$ \scriptsize{[Thm~\ref{thm:intro_lower}]}
                            & $O(\opt)$ \scriptsize{[Thm~\ref{thm:intro-upper}]} \\
    \hline
  \end{tabular}
\end{table}

\subsection{Overview of Techniques} \label{ssec:techniques}
Our approach is inspired by the recent work \citep{DKTZ20}, where the authors
use the same loss function for learning halfspaces under the (weaker)
Massart noise model. Under similar distributional assumptions to the ones
we consider here, \citep{DKTZ20} shows that the gradient of the loss function
points towards the parameters of the optimal halfspace.
A major difference between the two settings is that under Massart
noise there exists a \emph{unique optimal halfspace} and is identifiable.  In
the agnostic setting, there may be multiple halfspaces achieving optimal error.
However, as we show, for the class of distributions we consider, all these
solutions lie on a small cone, see Claim~\ref{clm:angle-zero-one}
establishing that the angle between any two halfspaces is small.
Our algorithm aims to move towards the cone with every gradient step.

To achieve this, we must carefully set the parameter $\sigma$ of the objective. Smaller
values of sigma amplify the contribution to the gradient of points closer to
the current guess and enable using local information to obtain good gradients.
This localization approach is necessary and is commonly used to efficiently
learn halfspaces under structured distributions~\cite{ABL17,DKS18a}. In the
Massart model, the authors of \citep{DKTZ20} show that for the loss function of
Equation~\eqref{eq:intro_objective} any sufficiently small value for $\sigma$
suffices to obtain a gradient pointing towards the optimal halfspace. This is
not true in the agnostic setting that we consider here.  In particular,
choosing small values of $\sigma$ may put a lot of weight on points close to
the halfspace that may all be noisy. To prove our structural result,
we show that there exists an appropriate setting of a not-too-small $\sigma$
that will guarantee convergence to a solution with $O(\opt)$ error.
This is our main structural result, Lemma~\ref{lem:structural_agnostic}.

Our lower bound hinges on the fact that
such a trade-off can only be achieved using non-convex loss functions. In particular,
our lower bound construction leverages the structure of convex objectives to design a noisy
distribution where any convex objective results in misclassification error
$\omega(\opt)$. In more detail, we exploit the fact that all optimal halfspaces lie in a small
cone, and show that there exists a fixed noise distribution such that all convex
loss functions have non-zero gradients inside this cone.

\subsection{Related Work}\label{ssec:related}
Here we provide a detailed summary of the most relevant prior work with a focus
on $\poly(d/\eps)$ time algorithms.
\cite{KLS:09jmlr} studied the problem of learning homogeneous
halfspaces in the adversarial label noise model, when the marginal
distribution on the examples is isotropic log-concave, and gave a
polynomial-time algorithm with error guarantee
$\tilde{O}(\opt^{1/3})+\eps$. This error bound was improved
by~\cite{ABL17} who gave an efficient localization-based
algorithm that learns to accuracy $O(\opt)+\eps$ for isotropic
log-concave distributions. \cite{DKS18a} gave a localization-based
algorithm that learns arbitrary halfspaces with error $O(\opt)+\eps$
for Gaussian marginals.  \cite{BZ17} extended the algorithms
of~\cite{ABL17} to the class of $s$-concave distributions, for
$s>-\Omega(1/d)$.  Inspired by the localization approach,
\cite{YanZ17} gave a perceptron-like learning algorithm that succeeds
under the uniform distribution on the sphere. The algorithm of
\cite{YanZ17} takes $\tilde{O}(d/\eps)$ samples, runs in time
$\tilde{O}(d^2/\eps)$, and achieves error of $\tilde{O}(\log d \cdot
\opt )+ \eps$  -- scaling logarithmically with the dimension $d$.
We also note that~\cite{DKTZ20} established a structural result
regarding the sufficiency of stationary points for learning homogeneous
halfspaces with Massart noise.
Finally, we draw an analogy with recent work \cite{DGKKS20} which 
established that convex surrogates suffice to obtain error $O(\opt)+\eps$ 
for the related problem of agnostically learning ReLUs under well-behaved distributions. 
This positive result for ReLUs stands in sharp contrast to the case of sign activations studied in this paper
(as follows from our lower bound result).
An interesting direction is to explore the effect of non-convexity
for other common activation functions.

 \newcommand{\capfun}{\mathrm{cap}}
\newcommand{\CLR}{\mathrm{CappedLeakyRelu}}

\section{Preliminaries and Notation}

For $n \in \Z_+$, let $[n] \eqdef \{1, \ldots, n\}$.  We will use small
boldface characters for vectors.  For $\bx \in \R^d$ and $i \in [d]$, $\bx_i$
denotes the $i$-th coordinate of $\bx$, and $\|\bx\|_2 \eqdef
(\littlesum_{i=1}^d \bx_i^2)^{1/2}$ denotes the $\ell_2$-norm of $\bx$.  We
will use $\langle \bx, \by \rangle$ for the inner product of $\bx, \by \in
\R^d$ and $ \theta(\bx, \by)$ for the angle between $\bx, \by$. We will also
denote $\1_A$ to be the characteristic function of the set $A$, i.e.,
$\1_A(\x)= 1$ if $\x\in A$ and $\1_A(\x)= 0$ if $\x\notin A$.  Let $\vec e_i$
be the $i$-th standard basis vector in $\R^d$.  Let $\proj_U(\vec x)$ be the
projection of $\vec x$ onto subspace $U \subset \R^d$.  Let $\E[X]$ denote the
expectation of random variable $X$ and $\pr[\mathcal{E}]$ the probability of
event $\mathcal{E}$.
We consider the binary classification setting where labeled
examples $(\bx,y)$ are drawn i.i.d. from a distribution $\D$ on $\R^d \times \{
\pm 1\}$.  We denote by $\D_{\bx}$ the marginal of $\D$ on $\vec x$.  The
misclassification error of a hypothesis $h: \R^d \to \{\pm 1\}$ (with respect
to $\D$) is $\err_{0-1}^{\D}(h) \eqdef \pr_{(\bx, y) \sim \D}[h(\bx) \neq y]$.
The zero-one error between two functions $f, h$ (with respect to $\D_{\bx}$) is
$\err_{0-1}^{\D_{\bx}}(f, h) \eqdef \pr_{\bx \sim \D_{\bx}}[f(\bx) \neq
h(\bx)]$.
 \section{Non-Convex SGD Learns Halfspaces with Adversarial Noise} \label{section3}\label{sec:struct}
In this section, we prove our main algorithmic result, whose formal
version we restate here.
\begin{theorem}\label{thm:main-inf}
	Let $\D$ be a distribution on $\R^d \times \{\pm1\}$ such that the
	marginal $\D_{\bx}$ on $\R^d$ is well-behaved. There is an algorithm
	with the following performance guarantee: For any $\eps>0$, it draws $m
	=\widetilde{O}(d\log(1/\delta)/\eps^4 ) $
	labeled examples from $\D$, uses $O(m)$ gradient evaluations, and
	outputs a hypothesis vector $\bar{\vec w}$ that satisfies
	$\err_{0-1}^{\D}(h_{\bar{\bw}})\leq  O(\opt) + \eps$ with probability
	at least $1-\delta$, where $\opt$ is the minimum classification error achieved by halfspaces.
\end{theorem}
The crucial component in the proof of Theorem~\ref{thm:main-inf} is the
following structural lemma, Lemma~\ref{lem:structural_agnostic}.
We show that by carefully choosing the parameter $\sigma>0$ of the non-convex surrogate loss
$S_\sigma$ of Equation~\eqref{eq:intro_objective},  we get that any approximate
stationary point of this objective will be close to some optimal halfspace.
Instead of optimizing over the unit sphere, we can normalize our objective $\SL$
defined in Equation~\eqref{eq:intro_objective}, as follows
\begin{equation} \label{eq:surr}
\SL(\vec w)  =
\E_{(\vec x, y) \sim \D}
\left[ S_{\sigma}\left(-y \frac{\dotp{\vec w}{\vec x}}{\snorm{2}{\vec w}} \right)\right],
\end{equation}
where $S_\sigma(t) = \frac{1}{1 + e^{-t/\sigma}}$ is the logistic
function with growth rate $1/\sigma$.  We prove the following:

\begin{lemma}[Stationary points of $\SL$ suffice]\label{lem:structural_agnostic}
  Let $\D_{\bx}$ be a well-behaved distribution on $\R^d$ and let $\wstar$
  be a halfspace achieving optimal classification error $\opt$.
  Fix $\sigma > 0$ and let $\theta = (4 \sqrt{2} \pi U/R) \cdot \sigma$.
  If $\opt \leq R^4/(2^{15} U^3)  \cdot \sigma$, then for every $\wh{\bw}$ such that
  $\theta(\wh{\bw}, \wstar) \in (\theta, \pi-\theta)$ it holds
$\snorm{2} {\nabla_{\bw} \SL(\wh{\bw})} \geq \frac{R^2}{64U} $.
\end{lemma}
\begin{proof}
	To simplify notation, we will write $h(\bw, \bx) = \frac{\dotp{\bw}{\bx}}{\snorm{2}{\bw}}$.
	Note that $\nabla_{\bw} h(\bw, \bx) = \frac{\bx}{\snorm{2}{\bw}} - \dotp{\bw}{\bx} \frac{\bw}{\snorm{2}{\bw}^3}$.
	We define the ``noisy" region $S$, as follows $ S=\{\x\in \R^d: y\neq \sign(\dotp{\vec w^*}{\vec x}) \}$.
	The gradient of the objective $\SL(\bw)$ is then
	\begin{align*}
	\nabla_{\vec w} \SL(\vec w)
	&= \E_{(\bx,y) \sim \D} \left[ - S_{\sigma}' \left(-y\ h(\bw, \bx)\right) \nabla_{\bw} h(\bw, \bx) \ y \right]\\
	& = \E_{(\bx, y) \sim \D} \left[- S'_{\sigma}\left(| h(\bw, \bx) | \right) \ \nabla_{\bw} h(\bw, \bx) \ y \right] \\
	& = \E_{\bx \sim \D_{\bx}} \left[- S'_{\sigma}\left(| h(\bw, \bx) | \right) \ \nabla_{\bw} h(\bw, \bx) \ (\1_{S^c}(\bx) - \1_S(\bx)) \ \sign(\dotp{\wstar}{\vec x}) \right]   \\
	& = \E_{\bx \sim \D_{\bx}} \left[- S'_{\sigma}\left(| h(\bw, \bx) | \right) \ \nabla_{\bw} h(\bw, \bx) \ (1 - 2 \cdot \1_S(\bx)) \ \sign(\dotp{\wstar}{\vec x}) \right] \;.
	\end{align*}
	Let  $V = \mrm{span}(\bw^{\ast}, \bw)$.  Since projections can only decrease
	the norm of a vector, we have
	$
	\snorm{2}{\nabla_{\bw}\SL(\bw)} \geq
	\snorm{2}{\proj_{V} \nabla_{\bw} \SL(\bw)}
	\,.
	$
	Without loss of generality, we may assume that $\wh{\bw} = \vec e_2$ and
	$\wstar = -\sin\theta \cdot \vec e_1 + \cos \theta \cdot \vec e_2$.  Then,
	we have $\proj_V(h(\bw, \bx)) = (\bx_1, 0)$.
	Using the above and the triangle inequality, we obtain
	\begin{align*}
	\snorm{2}{\nabla_{\vec w} \SL(\vec w)}
	\geq
	&\underbrace{\snorm{2}{\E_{\bx \sim \D_{\bx}} \left[- S'_{\sigma}\left(| h(\bw, \bx) | \right) ~ (\bx_1, 0) ~ \sign(\dotp{\wstar}{\vec x}) \right] }}_{I_1}
	\\
	&-
	2 \underbrace{\snorm{2}{\E_{\bx \sim \D_{\bx}} \left[- \1_{S}(\bx) S'_{\sigma}\left(| h(\bw, \bx) | \right) ~ (\bx_1, 0) ~ \sign(\dotp{\wstar}{\vec x}) \right] }}_{I_2}\;.
	\end{align*}
	Let $R, U$ be absolute constants from the Definition~\ref{def:bounds}.  We will
	first bound from above the term $I_2$, i.e., the contribution of the noisy
	points to the gradient.  Using the fact that $S_\sigma'(|t|) \leq
	e^{-|t|/\sigma}/\sigma$ we obtain
	\begin{align}
	I_2 \leq
	\E_{\bx \sim\D_{\bx}}
	\left[
	\frac{e^{-|\bx_2|/\sigma}}{\sigma}\ |\bx_1|\ \1_{S}(\bx)\right]
	&\leq \sqrt{\E_{\bx \sim \D_{\bx}} \left[ \1_{S}(\bx)\right]}
	\sqrt{\E_{\bx \sim \D_{\bx}}
		\left[\frac{e^{-2 |\bx_2|/\sigma}}{\sigma^2}\ \bx_1^2\right]} \nonumber \\
	&\leq  \sqrt{\frac{\opt}{\sigma}}
	\sqrt{
		\E_{\bx \sim(\D_{\bx})_V}
		\left[\frac{e^{- 2 |\bx_2|/\sigma}}{\sigma}\ \bx_1^2 \right]}
	\nonumber
	\,,
	\end{align}
	where the first inequality follows from the Cauchy-Schwarz inequality and for the
	second we used the fact that the set $S$ has probability at most $\opt$.  To
	finish the bound, notice that the remaining expectation depends only on $\bx_1,
	\bx_2$ and therefore we can use the upper bound $t(\cdot)$ on the density
	function.  Using polar coordinates we obtain
	\begin{align*}
	\E_{\bx \sim(\D_{\bx})_V}
	\left[\frac{e^{- 2 |\bx_2|/\sigma}}{\sigma}\ \bx_1^2 \right]
	&\leq
	4 \int_{0}^{\infty}\int_{0}^{\pi/2} \frac{r^3}{\sigma}
	\cos^2(\phi)  e^{-2 r\sin(\phi)/\sigma} t(r) \d \phi \d r
	\\
	&\leq
	2 \int_{0}^{\infty} r^2 t(r)
	\int_{0}^{\pi/2} \frac{2 r}{\sigma} \cos(\phi)
	e^{-2 r\sin(\phi)/\sigma} \d \phi \d r
	\\
	&=
	2 \int_{0}^{\infty} r^2 t(r) (1- e^{-2 r/\sigma}) \d r
	\leq 2 \int_{0}^{\infty} r^2 ~ t(r) \d r \leq 2 U\,,
	\end{align*}
	where for the last inequality we used the fact that $1-e^{-2 r/\sigma}\leq 1$.
	We thus have $I_2 \leq \sqrt{2 U \opt/ \sigma}$.

We now bound $I_1$ from below.  Observe that since inner
	products with $\wstar$, $\bw$ are preserved when we project
	$\bx$ to $V$, we have
	$
	I_1 =  \Big|    \E_{\bx \sim(\D_{\bx})_V}[ S_{\sigma}'(|\bx_2|) \bx_1 \sgn(\dotp{\wstar}{\bx})] \Big|
	$.  Now, if we define $G = \{(\bx_1, \bx_2) \in \R^2 : \bx_1 \sgn(\dotp{\vec w^{\ast}}{\vec x}) > 0 \} $,
	using the triangle inequality we have
	\[
	I_1 \geq
	\E_{\bx \sim(\D_{\bx})_V}[ S_{\sigma}'(|\bx_2|) |\bx_1| \1_G(\bx)]
	-
	\E_{\bx \sim(\D_{\bx})_V}[ S_{\sigma}'(|\bx_2|) |\bx_1|  \1_{G^c}(\bx)]\;.
	\]

Moreover, using the fact that $ e^{-|t|/\sigma}/(4\sigma) \geq
	S'_{\sigma}(|t|) \leq e^{-|t|/\sigma}/\sigma $
	we get
	\begin{align}
	\label{eq:2d_gradient_difference_lower_bound}
	I_1
	&\geq  \frac{1}{4}
	\E_{\bx \sim (\D_{\bx})_V} \left[ |\bx_1| \1_{G}(\bx) e^{-|\vec x_2|/\sigma }/\sigma \right]
	- \E_{\bx \sim (\D_{\bx})_V} \left[|\bx_1| \1_{G^c}(\bx) e^{-|\vec x_2|/\sigma}/\sigma \right] \;.
	\end{align}
	We can now bound each term separately
	using the fact that the distribution $\D_\bx$ is well-behaved.
	Assume first that $\theta(\wstar, \wh{\bw}) = \theta \in (0, \pi/2)$.
	Then we can express the region $G$ in polar coordinates as
	$G = \{ (r, \phi) : \phi \in (0, \theta) \cup (\pi/2, \pi +\theta) \cup (3 \pi/2, 2 \pi) \}$.

	We denote by $\gamma(x, y)$ the density of the $2$-dimensional
	projection on $V$ of the marginal distribution $\D_{\bx}$.  Since the integral
	is non-negative, we can bound from below the contribution of region $G$ on
	the gradient by integrating over $\phi \in (\pi/2, \pi)$. Specifically, we have:
	\begin{align}
	\E_{\bx \sim(\D_{\bx})_V} \left[\frac{e^{-|\bx_2|/\sigma}}{\sigma}\ |\bx_1|\
	\1_{G}(\bx)\right]
	&\geq \int_{0}^{\infty} \int_{\pi/2}^{\pi} \gamma(r \cos\phi,r \sin\phi)r^2 |\cos\phi|  \frac{\sigMO{r \sin{\phi}}}{\sigma} \d\phi \d r \nonumber\\
	&= \int_{0}^{\infty} \int_{0}^{\pi/2} \gamma(r \cos\phi,r \sin\phi)r^2 \cos\phi  \frac{\sigMO{r \sin{\phi}}}{\sigma} \d\phi \d r \nonumber\\
	&\geq \frac{1}{U} \int_{0}^{R} r^2 \d r \int_0^{\pi/2} \cos\phi \frac{\sigMO{R \sin{\phi}}}{\sigma} \d\phi \nonumber   \\
	&= \frac{1}{3 U} R^2 \left(1-e^{-\frac{R}{\sigma }}\right)
	\geq \frac{1}{4 U} R^2
	\label{eq:sec3Good}\;,
	\end{align}
	where for the second inequality we used the lower bound $1/U$ on the density
	function $\gamma(x,y)$ (see Definition~\ref{def:bounds}) and for the last inequality we used that $\sigma \leq \frac{R}{8}$ and that $1-e^{-8}\geq 3/4$.

	We next bound from above the contribution of the gradient in region $G^c$.
	Note that $G^c = \{(r, \phi): \phi \in B_\theta = (\pi/2-\theta, \pi/2) \cup (3 \pi/2 -\theta, 3 \pi/2)\}$.
	Hence, we can write:
	\begin{align}
	\E_{\bx \sim(\D_{\bx})_V} \left[\frac{e^{-|\bx_2|/\sigma}}{\sigma}\ |\bx_1|\ \1_{G^c}(\bx)\right]
	&= \frac{1}{\sigma} \int_{0}^{\infty} \int_{\phi \in B_{\theta}} \density(r \cos \phi,r \sin \phi)r^2 \cos{\phi} \sigMO{r \sin{\phi}} \d\phi \d r\nonumber \\
	&\leq  \frac{2U}{\sigma}\int_{0}^{\infty}\int_{\theta}^{\pi/2} r^2 \cos{\phi} \sigMO{r \sin{\phi}} \d\phi \d r \nonumber\\
	&= \frac{2U \sigma ^2 \cos^2 \theta}{\sin^2 \theta}
\label{eq:sec3Bad} \;,
	\end{align}
	where the inequality follows from the upper bound $U$ on the density
	$\gamma(x,y)$ (see Definition~\ref{def:bounds}).  Putting everything in
	\eqref{eq:2d_gradient_difference_lower_bound}, we obtain $I_1 \geq R^2/(16 U)
	-2 U \sigma^2/\sin^2\theta$.  Notice now that the case where $\theta(\wh{\bw},
	\wstar) \in (\pi/2,\pi-\theta )$ follows similarly.  Finally, in the case where
	$\theta=\pi/2$, the region $G^c$ is empty, and we again get the same lower
	bound on the gradient.  Let $A>0$, and set $\theta= A\cdot \sigma < \pi/2$, and let
	$\tau = \opt/\sigma$.  Since $\sin(t) \geq 2 t/ \pi$ for every $t \in[0,\pi/2]$,
	we have
	\begin{align*}
	I_1 - 2 I_2 \geq \frac{R^2}{16 U} -\frac{\pi^2 U}{2 A^2} -2 \sqrt{2 U \tau} \;.
	\end{align*}
	For $\tau \leq \frac{R^4}{2^{15} U^3}$ and  $A \geq 4 \sqrt{2} \pi U/R$,
	it holds $I_1 - 2 I_2 \geq R^2/(32 U)$.
\end{proof}
Using Lemma~\ref{lem:structural_agnostic} we get our main algorithmic result.
Our algorithm proceeds by Projected Stochastic Gradient Descent (PSGD), with
projection on the $\ell_2$-unit sphere, to find an approximate stationary point
of our non-convex surrogate loss.  Since $\SL(\bw)$ is non-smooth for vectors
$\bw$ close to $\vec 0$, at each step, we project the update on the unit sphere
to avoid the region where the smoothness parameter is high.  We are going to
use the following result about the convergence of non-convex, smooth SGD on the
unit sphere.
\begin{lemma}[Lemma 4.2 and 4.3 of \cite{DKTZ20}]\label{lem:opt-basic}
	Let $\SL(\bw)$ be as in Equation~\eqref{eq:intro_objective}. After $T$
	iterations, where $T = \Theta(d \log(1/\delta)/(\sigma^4 \rho^4))$, the
	output $({\vec w}^{(1)}, \ldots, {\vec w}^{(T)})$ of
	Algorithm~\ref{alg:PSGD} satisfies $\min_{i=1,\ldots, T}
	\snorm{2}{\nabla_{\vec w}\SL(\vec w^{(i)})} \leq \rho \;, $ with
	probability at least $1-\delta$.
\end{lemma}

\begin{algorithm}[H]
	\caption{PSGD for $f(\bw) = \E_{\vec z\sim \D}[g(\vec z, \bw)]$}
	\label{alg:PSGD}
	\begin{algorithmic}[1]
		\Procedure{psgd}{$f, T, \beta$}
		\Comment{$f(\bw) = \E_{\vec z \sim \D}[g(\vec z, \bw)]$: loss,
			$T$: number of steps, $\beta$: step size.}
		\State ${\vec w}^{(0)} \gets \vec e_1$
		\State \textbf{for} $i = 1, \dots, T$ \textbf{do}
		\State \qquad Sample $\vec z^{(i)}$ from $\D$.
		\State\qquad  ${\vec v}^{(i)} \gets {\vec w}^{(i-1)} - \beta \nabla_{\vec w} g({\vec z}^{(i)}, {\vec w}^{(i-1)})$
		\State \qquad ${\vec w}^{(i)} \gets {\vec v}^{(i)}/\snorm{2}{{\vec v}^{(i)}}$
		\State  \textbf{return} $({\vec w}^{(1)}, \ldots, {\vec w}^{(T)})$.
		\EndProcedure
	\end{algorithmic}
\end{algorithm}

In order to relate the misclassification error of a candidate halfspace with
the angle that it forms with an optimal halfspace, we are going to use the
following claim that states that the disagreement error between two halfspaces
is $\Theta( \theta(\vec u, \vec v))$ under well-behaved distributions.
\begin{claim}\label{clm:angle-zero-one}
	Let $\D_{\bx}$ be a distribution on $\R^d$. Let $f\in \argmin_{g\in
		{\cal C}}\err_{0-1}^{\D}(g)$, where $\cal C$ is the class of halfspaces, then for any $\vec u\in \R^d$, it holds
	that $\err_{0-1}^{\D_\bx}(h_{\vec u},f)- \err_{0-1}^{\D}(f)\leq \err_{0-1}^{\D}(h_{\vec u})\leq \err_{0-1}^{\D}(f) +
	\err_{0-1}^{\D_\bx}(h_{\vec u},f)$. Moreover, if the distribution $\D_{\bx}$  is well-behaved, then 	$
	\err_{0-1}^{\D_{\bx}}(h_{\vec u},h_{\vec v})
	= \Theta( \theta(\vec u, \vec v))
	$.
\end{claim}
\begin{proof}
	Let $S=\{\x\in \R^d: y\neq f(\bx) \}$, then we have
	\begin{align*}
	\err_{0-1}^{\D_{\bx}}(h_{\vec u},f)
	&= \int_{S^c} \1\{h_{\vec u}(\bx)\neq y\} \gamma(\bx) \d\bx
	+ \int_{S} \1\{h_{\vec u}(\bx)= y\} \gamma(\bx) \d\bx \\
	&=\int_{\R^d} \1\{h_{\vec u}(\bx)\neq y\} \gamma(\bx) \d\bx
	+  2 \int_{S} \1\{h_{\vec u}(\bx)= y\} \gamma(\bx) \d\bx
	- \int_{S} \gamma(\bx) \d\bx \\
	&= \err_{0-1}^{\D}(h_{\vec u}) +
	2 \int_{S} \1\{h_{\vec u}(\bx)= y\} \gamma(\bx)\d\bx - \err_{0-1}^{\D}(f)\;.
	\end{align*}
	Using that $\int_{S} \1\{h_{\vec u}(\bx)= y\} \gamma(\bx)\d\bx \geq 0$, the result follows.
	To prove that  $\err_{0-1}^{\D_\bx}(h_{\vec u},f)- \err_{0-1}^{\D}(f)\leq \err_{0-1}^{\D}(h_{\vec u})$, we work as follows
	\begin{align*}
	\err_{0-1}^{\D_{\bx}}(h_{\vec u},f)
	&= \int_{S^c} \1\{h_{\vec u}(\bx)\neq y\} \gamma(\bx) \d\bx
	+ \int_{S} \1\{h_{\vec u}(\bx)= y\} \gamma(\bx) \d\bx \\
	&=\int_{\R^d} \1\{h_{\vec u}(\bx)\neq y\} \gamma(\bx) \d\bx + \int_{S} \gamma(\bx) \d\bx
	-  2 \int_{S} \1\{h_{\vec u}(\bx)\neq y\} \gamma(\bx) \d\bx \\
	&= \err_{0-1}^{\D}(h_{\vec u}) + \err_{0-1}^{\D}(f)-
	2 \int_{S} \1\{h_{\vec u}(\bx)\neq y\} \gamma(\bx) \d\bx \;.
	\end{align*}
	To finish the proof, note that $\int_{S} \1\{h_{\vec u}(\bx)\neq y\} \gamma(\bx) \d\bx\geq 0$.
\end{proof}
Now assuming that we know the value of $\opt$, we can readily use SGD and obtain a
halfspace with small classification error.  The following lemma, which relies on
Claim~\ref{clm:angle-zero-one}, shows that SGD will output a list of
candidate vectors, one of which will have error $\opt + O(\sigma)$.  For our
structural result to work, we need $\opt \leq C \sigma$ which gives the
$O(\opt)$ error overall.  Recall that for all well-behaved distributions the parameters $U, R$
are absolute constants.
\begin{lemma}\label{lem:opt-helper}
	Let $\D$ be a distribution on $\R^d \times \{\pm1\}$ such that the
	marginal $\D_{\bx}$ on $\R^d$ is well-behaved. Algorithm~\ref{alg:PSGD}
	has the following performance guarantee: If $ \opt \leq C\cdot \sigma$
	where $C= \frac{ R^4}{ 2^{15} U^3} $, it draws $m = \poly(U/R) \cdot
	d\frac{\log(1/\delta)}{\sigma^4 } $ labeled examples from $\D$, uses
	$O(m)$ gradient evaluations, and outputs a hypothesis list of vectors
	$L$, such that there exists a vector $\bar{\vec w}\in L$ that satisfies
	$\err_{0-1}^{\D}(h_{\bar{\bw}})\leq \opt + O(\sigma)$ with probability
	at least $1-\delta$, where $\opt$ is the minimum classification error achieved by halfspaces.
\end{lemma}
\begin{proof}
	Let $R, U$ be the absolute constants from the
	Definition~\ref{def:bounds}. If we set $\rho=\frac{R^2}{32 U}$,
	by Claim~\ref{clm:angle-zero-one}, to guarantee
	$\err_{0-1}^{\D_{\bx}}(h_{\bar{\vec w}},f) \leq \sigma$
	it suffices to show that the angle $\theta(\bar{\vec w}, \vec \wstar) \leq O(\sigma ) =: \theta_0$.
	Using (the contrapositive of) Lemma~\ref{lem:structural_agnostic},
	if the norm squared of the gradient of some vector $\vec w \in \mathbb{S}^{d-1}$
	is smaller than $\rho$, then $\vec w$
	is close to either $\vec \wstar$ or $-\vec \wstar$ -- that is, $\theta(\vec w, \vec \wstar) \leq \theta_0$ --
	or $\theta(\vec w, -\vec \wstar) \leq \theta_0$.
	Therefore, it suffices to find a point $\vec w$ with gradient
	$\snorm{2}{\nabla_{\vec w} \SL(\vec w)} \leq \rho$.
	From Lemma~\ref{lem:opt-basic}, after $T= O(\frac{d}{\sigma^4 \rho^4} \log(1/\delta))$ steps,
	the norm of the gradient of some vector in the list $({\vec w}^{(0)}, \ldots, {\vec w}^{(T)})$
	will be at most $ \rho$ with probability $1-\delta$.  Therefore,
	the required number of iterations is
	$T= \poly(U/R) \cdot d\frac{\log(1/\delta)}{\sigma^4 }$.
	Note that one of the hypotheses in the list that is returned by Algorithm~\ref{alg:PSGD}
	is $\sigma$-close to the true $\vec \wstar$. From
	Claim~\ref{clm:angle-zero-one}, we have that there exists a
	$\hat{\vec w}\in  L$ such that $\err_{0-1}^{\D}(h_{\hat{\vec w}})\leq
	\opt + O(\sigma) = \opt + O(\sigma)$.
\end{proof}
We now give the proof of our main theorem, Theorem~\ref{thm:intro-upper}.
\begin{proof}[Proof of Theorem~\ref{thm:intro-upper}]
        Let $R, U$ be the absolute constants from Definition~\ref{def:bounds}.
        and let $C = 2^{15} U^3/R^4$.  We will do binary search to find the
        correct value of $\sigma$ using a grid of size $O(1/\eps)$.  In
        particular, we consider $\sigma \in \{C \eps, (C+1) \eps, \ldots, C \}
        $.  We now analyze our binary search over this grid.  We have three
        cases.  We first assume that $\eps \leq \opt \leq C$.
        Let $L_k$ be the list of candidates output
        by Algorithm~\ref{alg:PSGD} for  $\sigma=k \cdot \eps$.  Note
        that there is a value of $k$ such that $\opt< C \sigma$ and $\opt> C
        \sigma - \eps$. Then we have that there exists $\hat{\vec w} \in L_k$
        such that $\err_{0-1}(h_{\hat{\vec w}})\leq \opt + O(\sigma) = O(\opt)
        + \eps$.  To find the right value of $k$, we do binary search in the
        $O(1/\eps)$-sized grid of possible values and check each time if we
        obtained a weight vector that decreased the overall error.  Thus, we
        will overall construct $\poly(R/U)\cdot \log(1/\eps)$ lists.  Finally,
        to evaluate all the vectors from the list, we need a small number of
        samples from the distribution $\D$ to obtain the best among them, i.e.,
        the one that minimizes the zero-one loss.  The maximum size of each
        list of candidates is $\poly(U/R) \cdot d \frac{\log(1/\delta)}{\eps^4}$,
        Therefore, from Hoeffding's inequality,
        it follows that $ O(\log (d/(\eps \delta))/\eps^2)$ samples are
        sufficient to guarantee that the excess error of the chosen hypothesis
        is at most $\eps$ with probability at least $1-\delta$.
        Similarly, in the case where $\opt \leq \eps$ we have that for
        $\sigma= C \eps$, by running Algorithm~\ref{alg:PSGD}, we obtain a list
        $L_1$ of candidates.  From Lemma~\ref{lem:opt-helper}, we get that
        there is a vector $\hat{\vec 	w}\in L_1$, such that
        $\err_{0-1}(h_{\hat{\vec w}})\leq \opt + O(\sigma) \leq O(\eps)$.
        If $\opt \geq C$ then any halfspace will have error
        $\err_{0-1}(h_{\hat{\vec w}})\leq \poly(R/U) = O(\opt)$.
        We conclude that the total number of samples will be $\widetilde{O}(d\log(1/\delta)/\eps^4 )$.
        This completes the proof.
\end{proof}
 \section{Convex Objectives Do Not Work}\label{sec:convex_lower_bound}
      In this section, we show that optimizing convex surrogates of the zero-one
      loss cannot get error $O(\opt) + \eps$. We first recall the agnostic PAC
      learning setting that we assume here. Given a distribution $\D_{\bx}$ on
      $\R^d$ and a halfspace $\wstar$, we can define a noiseless instance $\D$
      on $\R^d \times \{\pm 1\}$ by setting the label of each point $\bx$ to $y
      = \sgn(\dotp{\wstar}{\bx})$.  In this setting, $\wstar$ achieves $0$
      classification error.  To get a distribution where $\wstar$ achieves
      error $\opt>0$, we can simply flip the labels of an $\opt$ fraction of
      points $\bx$.  In this section, we show that optimizing convex surrogates
      of the zero-one loss cannot get error $O(\opt) + \eps$, even under
      Gaussian marginals.  Recall that we consider objectives of the form
      \begin{equation}
        \mathcal{C}(\bw) = \E_{\bx, y \sim \D}[\ell(-y \dotp{\bx}{\bw})]\;,
      \end{equation}
      where $\ell(\cdot)$ is a convex loss function. Notice that by considering  the \emph{population} version of
      the objective in Equation~\eqref{eq:intro_convex_objective}, we
      essentially rule out the possibility of sampling errors to be the reason
      that the minimizer of the convex objective did not achieve low
      classification error.  With standard tools from empirical processes, one
      can readily get the same result for the empirical objective $(1/N)
      \sum_{i=1}^N \ell(-y^{(i)} \dotp{\bx^{(i)}}{\bw})$ assuming that the
      sample size $N$ is sufficiently large.  We now restate the main result of
      this section that allows us to show Theorem~\ref{thm:intro_lower}.
      \definecolor{mycyan}{RGB}{42,161,152}
\definecolor{myred}{RGB}{220,50,47}
\definecolor{gray}{RGB}{70,70,70}

\begin{figure}
	\centering
	\begin{tikzpicture}[scale=0.6, every node/.style={transform shape} ]
	\coordinate (start) at (0.5,0.5);
	\coordinate (center) at (0,0);
	\coordinate (end) at (0,0.7);
	\coordinate (end2) at (-0.12/0.7634,0.366/0.7634);
	\coordinate (start2) at (0,0.7);
	\coordinate (start3) at (3,0);
	\coordinate (end3) at (15.3:3);
	\coordinate (start4) at (135:2);
	\coordinate (end4) at (-2,0);
	\draw (0,0) circle (3);
	
	\draw[fill=myred, opacity=0.6] (-3,0) -- (-4.3,0) arc (180:90:4.3cm) -- (0,3) arc (90:180:3cm);
	\draw[pattern=north west lines, pattern color=black, opacity=1] (-3,0) -- (-4.3,0) arc (180:90:4.3cm) -- (0,3) arc (90:180:3cm);
	
	\draw[fill=myred, opacity=0.6] (0,3) -- (0,4.3) arc (90:45:4.3cm) -- (2.82*3/4,2.82*3/4) arc (45:90:3cm);
	\draw[pattern=north west lines, pattern color=black, opacity=1] (0,3) -- (0,4.3) arc (90:45:4.3cm) -- (2.82*3/4,2.82*3/4) arc (45:90:3cm);
	
	\draw[fill=myred, opacity=0.6] (-3,0) -- (-4.3,0) arc (180:225.3:4.3cm) -- (-2.115,-2.115) arc (225:180:3cm);
	\draw[pattern=north west lines, pattern color=black, opacity=1] (-3,0) -- (-4.3,0) arc (180:195.3:4.3cm) -- (195.3:3) arc (195.3:180:3cm);
	
	\draw[fill=myred, opacity=0.6] (0,0) -- (3,0) arc (360:195.3:3.0cm) -- cycle;
	\draw[fill=myred, opacity=0.6] (0,0) -- (3,0) arc (0:15.3:3.0cm) -- cycle;
	
	\draw[fill=mycyan, opacity=0.6] (3,0) -- (4.3,0) arc (0:45.3:4.3cm) -- (2.82*3/4,2.82*3/4) arc (45:0:3cm);
	\draw[pattern=north west lines, pattern color=black, opacity=1] (3,0) -- (4.3,0) arc (0:15.3:4.3cm) -- (15.3:3) arc (15.3:0:3cm);
	
	\draw[fill=mycyan, opacity=0.6] (3,0) -- (4.3,0) arc (360:270:4.3cm) -- (0,-3) arc (270:360:3cm);
	\draw[pattern=north west lines, pattern color=black, opacity=1] (3,0) -- (4.3,0) arc (360:270:4.3cm) -- (0,-3) arc (270:360:3cm);
	
	\draw[fill=mycyan, opacity=0.6] (0,-3) -- (0,-4.3) arc (-90:-135:4.3cm) -- (-2.82*3/4,-2.82*3/4) arc (-135:-90:3cm);
	\draw[pattern=north west lines, pattern color=black, opacity=1] (0,-3) -- (0,-4.3) arc (-90:-135:4.3cm) -- (-2.82*3/4,-2.82*3/4) arc (-135:-90:3cm);
	
	\draw[fill=mycyan, opacity=0.6] (0,0) -- (-3,0) arc (180:15.3:3.0cm) -- cycle;
	\draw[fill=mycyan, opacity=0.6] (0,0) -- (-3,0) arc (180:195.3:3.0cm) -- cycle;
	
	\draw[black,thick,-] (-4.3,0) -- (4.3,0) ;
	\draw[->] (-4.6,0) -- (4.6,0) node[anchor=north west,black] {$\vec e_1$};
	\draw[->] (0,-4.6) -- (0,4.6) node[anchor=south east] {$\vec e_2$};
\draw[purple,thick,-] (3.2*0.95,-3.2*0.95) -- (-3.2*0.95,3.2*0.95) node[anchor= south east] {};
	\draw[yellow,thick,-] (-4.4*0.94,-1.2*0.94) -- (4.4*0.94,1.2*0.94) node[anchor= south east] {};
	\draw[yellow,thick,->] (0,-0) -- (-0.22/0.7634,0.732/0.7634)node[anchor= south ] {$\wstar$};
	\draw[purple,thick,->] (0,-0) -- (0.707,0.707)node[anchor= south east,above] {$\wt{\vec w}$};
	\draw[black,dashed] (-0.707*4.7,-0.707*4.7) -- (0.707*4.7,0.707*4.7);
	\draw[black,thick ,->] (0,0) -- (0,1) node[above right] {$\bw$};
\pic [draw, <->, angle radius=18mm, angle eccentricity=1.2, "$\theta_1$"] {angle = start3--center--end3};
	\pic [draw, <->, angle radius=15mm, angle eccentricity=1.2, "$\theta_2$"] {angle = start4--center--end4};
\node[] at (-2,-1.3) {$C$};
\node[] at (2,1.3) {$C$};
	\node[] at (2*1.414*3/4,2*1.414*3/4)[below] {$Z$};
\end{tikzpicture}
	\caption{The green region depicts all points with $+1$ label and the red
		region depicts points with $-1$ label.  We have $y =
		-\sgn(\dotp{\wstar}{\bx})$ for all points in $S \setminus C$, this corresponds to the
		hatched region.  We have
		$\theta(\wstar, \bw) = \theta_1$ and $\theta(\wt{\bw}, \bw) = \theta_2$.
	}
	\label{fig:lower_bound_noise}
\end{figure}
 \begin{theorem}
	\label{thm:main_lower_bound}
	Fix $Z > 0, \theta \in (0, \pi/8)$, and let $\D_{\bx}$ be a radially symmetric
	distribution on $\R^2$ such that
	\begin{enumerate}
		\item For all $t> 0$ it holds $\Prob_{\bx \sim \D_\bx}[\snorm{2}{\bx} \geq t] > 0$.
		\item
		$\E_{\bx \sim \D_\bx}\left[\1\{\snorm{2}{\bx} \geq Z\} \snorm{2}{\bx}\right]
		> 24 \theta ~ \E_{\bx \sim \D_\bx}\left[\snorm{2}{\bx}\right]
		$.
	\end{enumerate}
	Then there exists a distribution $\D$ on $\R^2 \times \{\pm 1 \}$ and a
	halfspace $\wstar$ such that
	$
	\err_{0-1}^{\D}(\wstar) \leq
	\Prob_{\bx \sim \D_\bx}[\snorm{2}{\bx} \geq Z]
	$,
	the $\bx$-marginal of $\D$ is $\D_{\bx}$, and for every convex, non-decreasing,
	non-constant loss $\ell(\cdot)$ and every $\bw$ such that $\theta(\bw,
	\wstar) \leq \theta$ it holds $\nabla_{\bw} \calC(\bw) \neq \vec 0$,
	where $\calC$ is defined in Eq.~\eqref{eq:intro_convex_objective}.
\end{theorem}

\begin{proof}
	We start by constructing the noisy distribution $\D$.
	Fix any unit vector $\wstar$ and let $\wt{\bw}$ be a vector
	such that $\theta(\wstar, \wt{\bw}) = \theta_2$, where $2 \theta \leq
	\theta_2 \leq \pi/4$.  Denote by $\wt{\bw}^\perp$ the vector that is
	perpendicular with $\wt{\bw}$ and satisfies $\dotp{\wstar}{\wt{\bw}^\perp} \geq 0$.
We now define the regions $C, S$ that will help us define the parts of
	the distribution where we will introduce noise by flipping the
	$y$-labels, see also Figure~\ref{fig:lower_bound_noise}.
	\begin{align*}
	C =
	\left\{\bx :~ \dotp{\wstar}{\bx} \dotp{\wt{\bw}}{\bx} \geq 0
	~
	\text{and}
	~
	\dotp{\wt{\bw}^{\perp}}{\bx}\leq 0
	\right\}
	~~~~~~~~
	S = \{\bx :~ \snorm{2}{\bx} \geq Z\} \,.
	\end{align*}
We are now ready to define our noisy distribution $\D$: \emph{we flip the
		labels of all points in the set $S \setminus C$.} Observe that
	$\err_{0-1}^{\D}(\wstar) \leq \Prob_{\bx \sim \D_\bx}[\snorm{2}{\bx} \geq Z]$.
	Take any $\bw$ such that $\theta_1 = \theta(\bw, \wstar) \leq \theta.$
	We are going to bound from below the norm of the gradient of $\calC$ at $\bw$.
	The gradient of $\calC(\vec w)$ is
	\[
	\nabla_\bw \calC(\vec w) =
	\E_{(\bx, y) \sim \D}
	[- y \bx ~ \ell'(-y \dotp{\bx}{\bw})].
	\]
	Without loss of generality, we may assume that $\bw  = \rho \vec e_2$,
	where $\rho = \snorm{2}{\bw} > 0$.  We have that the first coordinate
	of the gradient is
	\begin{equation}
	\label{eq:grad_e_1}
	\dotp{\nabla_\bw(\calC(\bw)}{\vec e_1}
	=
	\E_{(\bx, y) \sim \D} [- y \bx_1 ~ \ell'(-y \rho ~ \bx_2)] \,.
	\end{equation}
	In what follows, we are going to use polar coordinates $(r, \phi)$ with the
	standard relation to Cartesian $(\bx_1, \bx_2) = (r \cos \phi, r \sin \phi)$.
	Now assume that we want to compute the contribution of a specific
	region $A = \{r \in [r_1, r_2], \phi \in [\phi_1, \phi_2]\}$ to the
	gradient of
	Equation~\eqref{eq:grad_e_1}.  We denote the $2$-dimensional density of the
	radially symmetric distribution $\D_{\bx}$ as $\gamma(r)$.  We have
	\begin{align}
	\label{eq:polar_gradient}
	&\E_{(\bx, y) \sim \D} [- y \bx_1 ~ \ell'(-y \bx_2) \1_A(\bx)]
	=
	\int_{r_1}^{r_2} r \gamma(r) \int_{\phi_1}^{\phi_2}
	-y r \cos \phi ~ \ell'(- y \rho ~ r \sin \phi)
	\d \phi \d r
	\nonumber \\
	&=
	\frac{1}{\rho}
	\int_{r_1}^{r_2} r \gamma(r)
	\int_{\phi_1}^{\phi_2} (\ell(-y \rho r \sin \phi))' \d \phi
	\d r
	=
	\frac{1}{\rho}
	\int_{r_1}^{r_2} r \gamma(r)
	( \ell(-y \rho r \sin\phi_2)- \ell(-y \rho r \sin\phi_1))
	\d r \,.
	\end{align}
	Without loss of generality, we consider the two cases shown in
	Figure~\ref{fig:lower_bound_noise}.  We start with the first case, where $\bw$ lies
	between $\wstar$ and $\wt{\bw}$. We first compute the contribution to the
	gradient in $S^c$, i.e., the points where $y = \sgn(\dotp{\wstar}{\bx})$.
	Since the distribution is radially symmetric, we have
	$
	\E_{(\bx, y) \sim \D} [- y \bx_1 ~ \ell'(-y \bx_2)
	\1_{S^c}(\bx)]
	=
	2 \E_{(\bx, y) \sim \D} [- y \bx_1 ~ \ell'(-y \bx_2)
	\1_{R_1}(\bx)],
	$
	where $R_1 = \{r \in [0, Z], \phi \in [\theta_1, \pi + \theta_1]\}$.
	From Equation~\eqref{eq:polar_gradient}, we obtain that
	\[
	I_{S^c} = \E_{(\bx, y) \sim \D} [- y \bx_1 ~ \ell'(-y \bx_2)
	\1_{S^c}(\bx)]
	= \frac{2}{\rho} \int_0^Z r \gamma(r) (\ell(\rho r ~ \sin \theta_1) -
	\ell(-\rho r ~ \sin \theta_1)
	\d r \,.
	\]
	Observe that since $\ell(\cdot)$ is non-decreasing we have $I_{S^c}
	\geq 0$.  Next we compute the contribution of region $S$ to  the
	gradient.  Recall that $S$ contains $S\setminus C$, i.e., the region we flipped
	the labels, $y = - \sgn(\dotp{\wstar}{\bx})$, see
	Figure~\ref{fig:lower_bound_noise}.  Using again the fact that the
	distribution is
	radially symmetric and Equation~\eqref{eq:grad_e_1} for the region $R_2 = \{r
	\in [Z, +\infty), \phi \in [\pi/2 - \theta_2, 3\pi/2 - \theta_2]\}$, we
	obtain
	\begin{align*}
	I_{S} = \E_{(\bx, y) \sim \D} [- y \bx_1 ~ \ell'(-y \bx_2) \1_S(\bx)]
	&= \frac{2}{\rho} \int_Z^\infty r \gamma(r)
	\Big(\ell(\rho r ~ \sin(\frac{3\pi}2 - \theta_2)) -
	\ell(\rho r ~ \sin( \frac{\pi}2 - \theta_2)
	\Big)
	\d r \\
	&=
	\frac{2}{\rho} \int_Z^\infty r \gamma(r) (\ell(-\rho r ~ \cos\theta_2)
	-
	\ell(\rho r ~ \cos\theta_2) \d r \,.
	\end{align*}
	Similarly to the previous case, the fact that $\ell(\cdot)$ is non-decreasing
	implies that $I_{S} \leq 0$.

	Now we are going to use the facts that $\ell(\cdot)$ is convex and
	non-decreasing.  Since both $\theta_1, \theta_2 \leq \pi/4$, we have that
	$\cos\theta_2 \geq \sin\theta_1$ and therefore, from the convexity of
	$\ell(\cdot)$, we obtain
	\[
	\frac{\ell(\rho r \sin(\theta_1) ) - \ell(- \rho r \sin\theta_1)}{2 \rho r \sin\theta_1}
	\leq
	\frac{\ell(\rho r \cos \theta_2) - \ell(-\rho r \sin\theta_1)}
	{ \rho r \cos(\theta_2) + \rho r \sin(\theta_1)} \,.
	\]
	Since $\ell(\cdot)$ is also non-decreasing, we have that
	$\ell(\rho r \cos \theta_2) - \ell(-\rho r \sin\theta_1)
	\leq \ell(\rho r \cos \theta_2 ) - \ell(-\rho r \cos\theta_2)$ and therefore,
	\[
	\ell(\rho r \sin\theta_1 ) - \ell(-\rho r \sin\theta_1)
	\leq
	\frac{2 \sin\theta_1}{\cos\theta_2 +\sin\theta_1}
	(\ell(\rho r \cos \theta_2) - \ell(- \rho r \cos \theta_2 ))\;.
	\]
	To simplify notation, we define the functions
	$\bar{\ell}(r) = \ell(\rho r \cos \theta_2)$ and
	$h(r) = \bar{\ell}(r) - \bar{\ell}(-r)$.
	Observe that $\bar{\ell}(\cdot)$ enjoys exactly the same properties as
	$\ell(\cdot)$, that is $\bar{\ell}(\cdot)$ is convex, non-decreasing, and
	non-constant.  Moreover, observe that $h(r)$ is non-negative and
	non-decreasing.  Using the above inequalities, we obtain that
	\begin{align}
	\label{eq:gradient_with_h}
	\rho \dotp{\nabla_\bw\calC(\vec \bw)}{\vec e_1}
	=  \rho(I_{S} + I_{S^c})
	\leq
	\frac{4 \sin\theta_1}{\cos\theta_2 + \sin \theta_1}
	\underbrace{\int_0^Z r \gamma(r) h(r) \d r}_{I_2}
	- 2 \underbrace{\int_Z^{\infty} r \gamma(r) h(r) \d r}_{I_1} \,.
	\end{align}
	We will now show that instead of dealing with every convex and
	increasing $\bar{\ell}(\cdot)$, we can restrict our attention to simple
	piecewise-linear convex and increasing functions.  First, we observe
	that without loss of generality we may assume that the convex function
	$\bar{\ell}(r)$ is constant for all $r \leq -Z$, since that part only
	increases $I_1$.  To construct $s(\cdot)$, we use the supporting lines
	of $\bar{\ell}(\cdot)$ at $-Z$ and $0$, and the secant line from $0$ to
	$Z$.  We will first assume that $\bar{\ell}'(Z) > 0$.  Let $a_0$ be a
	subgradient of $\bar{\ell}(\cdot)$ at $0$.  Then the secant from $0$
	to $Z$ is some line $a_1 r - a_0 Z_0$ for some $a_1 \in [a_0,
	\bar{\ell}'(Z)]$.  Then, for every convex and non-decreasing
	$\bar{\ell}(\cdot)$, the following piecewise-linear function $s(r)$
	makes the ratio $I_1/I_2$ smaller.  In what follows, $Z_0 \in [-Z, 0]$
	is the intersection point of the supporting line $ a_0 r - a_0 Z_0$ and
	the constant supporting line at $-Z$.
	\begin{align*}
	s(r) = b +
	\begin{cases}
	0, &r \leq Z_0 \\
	a_0 r - a_0 Z_0 , &Z_0 < r \leq 0 \\
	a_1 r - a_0 Z_0, &0 < r
	\end{cases}\;.
	\end{align*}
	We have
	\begin{align*}
	h(r) =
	\begin{cases}
	(a_1 + a_0) r , &0 \leq r \leq -Z_0, \\
	a_1 r - a_0 Z_0 &-Z_0 < r
	\end{cases}\;.
	\end{align*}
	\begin{align*}
	I_1 = a_1 \int_Z^\infty r^2 \gamma(r) d r
	- a_0 Z_0 \int_{Z}^\infty r \gamma(r) d r
	\geq a_1 \int_Z^\infty r^2 \gamma(r) d r \,.
	\end{align*}
	\begin{align*}
	I_2 &= (a_1 + a_0) \int_0^{-Z_0} r^2 \gamma(r) d r
	+ a_1 \int_{-Z_0}^Z r^2 \gamma(r) d r - a_0 Z_0 \int_{-Z_0}^{Z} r \gamma(r) d r\\
	&\leq 2(a_1 + a_0) \int_{0}^{Z} r^2 \gamma(r) d r
	\leq 4 a_1 \int_{0}^{Z} r^2 \gamma(r) d r
	\,.
	\end{align*}
	Using the above bounds in Equation~\eqref{eq:gradient_with_h}, we obtain
	\[
	\dotp{\nabla_\bw\calC(\bw)}{\vec e_1}
	\leq
	\frac{2 a_1}{\rho}
	~
	\left(
	\frac{8 \sin\theta_1}{\cos\theta_2 + \sin \theta_1}
	\int_0^Z r^2 \gamma(r) \d r - \int_Z^\infty r^2 \gamma(r) \d r
	\right) \,.
	\]
	Removing the positive quantity $\sin \theta_1$ of the denominator and replacing
	$\theta_1$ by its upper bound $\theta$, we obtain the claimed bound.  Since
	$\cos \theta_2$ is decreasing in $[0, \pi/2]$, we may choose $\theta_2 = 2
	\theta$.  Our final bound is then
	\begin{align*}
	\dotp{\nabla_\bw\calC(\bw)}{\vec e_1}
	&\leq
	\frac{2 a_1}{\rho}
	~
	\left(
	8 \tan(2 \theta)
	\int_0^Z r^2 \gamma(r) \d r - \int_Z^\infty r^2 \gamma(r) \d r
	\right)
	\\
	&\leq
	\frac{2 a_1}{\rho}
	~
	\left(
	24 \theta ~ \E_{\bx \sim \D_\bx}[\snorm{2}{\bx}] -
	\E_{\bx \sim \D_\bx}[\1\{\snorm{2}{\bx} > Z\} \snorm{2}{\bx}]
	\right)\,,
	\end{align*}
	where for the last inequality we used the fact that $ \tan(2 \theta) \leq 3
	\theta$ for all $\theta \in [0,\pi/8)$.  In the case where $\ell'(\rho Z \cos
	\theta_2) = 0$, the above bound vanishes.  We fist assume that this is not the
	case.  Then, using Assumption 2 of our theorem, we obtain that
	$\dotp{\nabla_\bw\calC(\bw)}{\vec e_1} \neq 0$
	and therefore $\nabla_\bw\calC(\bw) \neq \vec 0$.

	In the case where $\ell'(\rho Z \cos \theta_2) = 0$, we observe that
	$I_{S^c}$ vanishes.  To finish the proof, we need to bound from above
	and away from zero the integral $I_{S}$.  Since $\bar \ell(\cdot)$ is
	non-constant, there exists a point $Z' > Z$ with $\bar{\ell}'(Z) > 0$.
	Convexity of $\bar{\ell}(\cdot)$ implies $h(r) \geq  \bar{\ell}'(Z) r
	$.  Using this fact, we get
	\[
	I_{S} \leq
	-\bar \ell'(Z') \int_{Z'}^\infty r^2 \gamma(r) \d r\,.
	\]
	Using Assumption 1
	of our theorem, we again get that $\nabla_\bw\calC(\bw) \neq \vec 0$.

	Next we handle the case where the candidate $\bw$ lies out of the cone
	formed by $\wstar$ and $\wt{\bw}$.  In that case, similarly to before,
	we compute the contribution to the gradient of the noisy
	samples $S$ and the non-noisy $S^c$.
	\[
	I_{S^c} = \E_{(\bx, y) \sim \D} [- y \bx_1 ~ \ell'(-y \bx_2)
	\1_{S^c}(\bx)]
	= \frac{2}{\rho} \int_0^Z r \gamma(r) (\ell(-\rho r ~ \sin \theta_1) -
	\ell(\rho r ~ \sin \theta_1)
	\d r \,.
	\]
	and
	\[
	I_{S} = \E_{(\bx, y) \sim \D} [- y \bx_1 ~ \ell'(-y \bx_2)
	\1_S(\bx)]
	= \frac{2}{\rho} \int_Z^\infty r \gamma(r) (\ell(-\rho r ~ \cos \theta_2) -
	\ell(\rho r ~ \cos \theta_2)
	\d r \,.
	\]
	In contrast to the previous case, we now observe that since $\ell(\cdot)$ is
	non-decreasing, both $I_S$ and $I_{S^c}$ have the same sign, i.e., they are
	both non-positive.  From Assumption 1, and the fact that $\ell(\cdot)$ is
	non-constant, we obtain that $I_S + I_{S^c} < 0$, which in turn
	implies that $\nabla_{\bw} \calC(\bw) \neq \vec 0$.
\end{proof}

We are now ready to give the proof of Theorem~\ref{thm:intro_lower}, which we restate
below for convenience.

\begin{customthm}{\ref{thm:intro_lower}}
	Let $\D_{\bx}$ be the standard normal distribution on $\R^d$.  There exists a
	distribution $\D$ on $\R^d \times \{\pm 1\}$ such that for \emph{every} convex,
	non-decreasing loss $\ell(\cdot)$, the objective $\calC(\bw) = \E_{\bx,
		y \sim \D}[\ell(-y \dotp{\bx}{\bw})]$ is minimized at some halfspace $h$
	with error $\err_{0-1}^{\D}(h) = \Omega(\opt \sqrt{\log(1/\opt)})$.
	Moreover, there exists a log-concave marginal $\D_{\bx}$ (resp.  $s$-heavy
	tailed marginal) such that
	$\err_{0-1}^{\D}(h) = \Omega(\opt \log(1/\opt))$
	(resp.
	$\err_{0-1}^{\D}(h) = \Omega(\opt^{1-1/s})$).
\end{customthm}
\begin{proof}
	Since all the examples that we are going to consider will be radially
	invariant distributions, we note that the ``disagreement'' error of two
	halfspaces with normal vectors $\vec v, \vec u$ is $\theta(\vec v, \vec
	u)/\pi$.  From Claim~\ref{clm:angle-zero-one}, we have that the
	classification error of any candidate $\bw$ is lower bounded by $\theta(\bw,
	\wstar)/\pi - \opt$.  We will construct a distribution $\D$ such that there
	is some $\wstar$ that achieves error $\opt$, but at the same
	time $\calC(\bw)$ is minimized at some halfspace such that
	$\theta(\bw, \wstar) = \omega(\opt)$.  This means that the minimizer
	of $\calC$ has classification error $\omega(\opt)$.

	We assume first that $\D_\bx$ is the standard normal and without loss of
	generality work in two dimensions.  Recall that the density function in this case is radially
	invariant, i.e., $\gamma(\bx_1, \bx_2) = \frac{1}{2 \pi}
	e^{-\snorm{2}{\bx}^2/2}$.
	If $\ell$ is a constant function, any halfspace would minimize
	it and therefore, this case is trivial.
	Clearly, Assumption 1 of
	Theorem~\ref{thm:main_lower_bound} holds in this case.  We now show that we
	can pick $Z > 0$ such that the probability of all points with flipped label
	is $O(\opt)$ and make Assumption 2 of Theorem~\ref{thm:main_lower_bound}
	true.  Since the distribution is Gaussian, we have that for $Z =
	\Theta(\sqrt{\log(1/\opt)})$ it holds $\pr[\snorm{2}{\bx} \geq Z] \leq \opt$.
	Since the distribution is isotropic, we have $\E_{\bx \sim
		\D_\bx}[\snorm{2}{\bx}] \leq \sqrt{\E_{\bx \sim \D_\bx}[\snorm{2}{\bx}^2]} =
	1$.
	Moreover, we have that
	\[
	\E_{\bx \sim \D_\bx}\left[\1\{\snorm{2}{\bx} \geq Z\} \snorm{2}{\bx}\right]
	=
	\int_Z^\infty r^2 e^{-r^2/2} \d r \geq e^{-Z^2/2} Z
	= \Theta(\opt \sqrt{\log(1/\opt)})\,.
	\]
	Now we can fix some $\theta = \Omega(\opt \sqrt{\log(1/\opt)}) < \pi/8$
	and observe that Assumption 2 of Theorem~\ref{thm:main_lower_bound} is
	satisfied.  Therefore, we have that for any halfspace with normal vector
	$\bw$ with $\theta(\bw, \wstar) \leq \theta = \Omega(\opt
	\sqrt{\log(1/\opt)})$  it holds that $\nabla_{\bw} \calC(\bw) \neq \vec 0$,
	and therefore it cannot be a minimizer of $\calC(\bw)$.

	For the log-concave marginals the argument is similar.  We work again in
	two dimensions and pick $\gamma(\bx) = \frac6\pi e^{-2 \sqrt{3}
		\snorm{2}{\bx}}$.  This distribution is isotropic log-concave.  We have
	that for $Z = \Theta(\log(1/\opt))$ it holds that $\Prob[\snorm{2}{\bx} \geq
	Z] \leq \opt$.  Moreover, we have
	$
	\E_{\bx \sim \D_\bx}\left[\1\{\snorm{2}{\bx} \geq Z\} \snorm{2}{\bx}\right]
= \Omega(\opt \log(1/\opt)).
	$

	Now we can fix some $\theta = \Omega(\opt \log(1/\opt)) < \pi/8$ and
	observe that Assumption 2 of Theorem~\ref{thm:main_lower_bound} is
	satisfied.  Therefore, we have that for any halfspace with normal vector
	$\bw$ with $\theta(\bw, \wstar) \leq \theta = \Omega(\opt \log(1/\opt))$
	it holds that $\nabla_{\bw} \calC(\bw) \neq \vec 0$, and as a result it
	cannot be a minimizer of $\calC(\bw)$.

	For the heavy tailed marginals, the argument is similar.  We work again in
	two dimensions, and for any $s > 2$ we pick
	\[
	\gamma(\bx)
	= \frac{b_s}{\left(\frac{\snorm{2}{\bx}}{a_s} + 1 \right)^{2 + s}} \,,
	\]
	where the constants $a_s, b_s$ depend only on $s>2$ and
	are appropriately picked so that the distribution is isotropic.
	Using polar coordinates,  we have
	\[
	\Prob[\snorm{2}{\bx} \geq Z]
	=
	2 \pi
	\int_Z^\infty \frac{r b_s}{\left(\frac{r}{a_s} + 1 \right)^{2 + s}}
	\d r
	=
	\frac{2 \pi b_s}{s (1+s)} \frac{a_s + (s+1) Z}{(a_s + Z)^{1+s}}\;.
	\]
	Therefore, for $Z = \Theta((1/\opt)^{1/s})$ it holds that $\Prob[\snorm{2}{\bx} \geq Z] \leq \opt$.
	Moreover, we have
	\[
	\E_{\bx \sim \D_\bx}\left[\1\{\snorm{2}{\bx} \geq Z\} \snorm{2}{\bx}\right]
	= 2 \pi
	\int_Z^\infty \frac{r^2 b_s}{\left(\frac{r}{a_s} + 1 \right)^{2 + s}}
	\d r
	= \frac{b_s \left(2 a_s^2+2 a_s (s+1) Z+s (s+1) Z^2\right)}{s \left(s^2-1 \right) (a_s+Z)^{s+1} }
	\,.
	\]
	Therefore, for $Z = \Theta((1/\opt)^{1/s})$ it holds
	$
	\E_{\bx \sim \D_\bx}\left[\1\{\snorm{2}{\bx} \geq Z\} \snorm{2}{\bx}\right]
	= \Omega(\opt^{1-1/s})
	$.
	We can now fix some $\theta = \Omega(\opt^{1-1/s}) < \pi/8$ and
	observe that Assumption 2 of Theorem~\ref{thm:main_lower_bound} is
	satisfied.  Therefore, we have that for any halfspace with normal vector
	$\bw$ with $\theta(\bw, \wstar) \leq \theta = \Omega(\opt^{1-1/s})$
	it holds that $\nabla_{\bw} \calC(\bw) \neq \vec 0$, and as a result it
	cannot be a minimizer of $\calC(\bw)$.
\end{proof}

\bibliographystyle{alpha}
\bibliography{allrefs}
\clearpage
\appendix

\end{document}